\documentclass[11pt, letterpaper]{article}

\usepackage{geometry}
\usepackage[utf8]{inputenc} 
\usepackage[T1]{fontenc}    

\usepackage{amsmath, amssymb, amsthm}
\newtheorem{theorem}{Theorem}

\usepackage{graphicx}
\usepackage{booktabs}       
\usepackage{multirow}       
\usepackage{xcolor}         

\usepackage{algorithm}
\usepackage{algorithmic}

\usepackage[numbers]{natbib} 
\usepackage{url}
\usepackage[colorlinks=true, citecolor=blue, linkcolor=blue, urlcolor=blue]{hyperref} 

\usepackage{pifont}
\newcommand{\cmark}{\ding{51}}   
\newcommand{\xmark}{\ding{55}}   
\newcommand{\pmark}{$\triangle$} 

\providecommand{\keywords}[1]{\vspace{0.5em}\noindent\textbf{\textit{Keywords---}} #1}

 \usepackage{authblk} 

\title{Operationalizing Fairness: Post-Hoc Threshold Optimization Under Hard Resource Limits}

\author[1]{Moirangthem Tiken Singh\thanks{Corresponding author. Email: \href{mailto:tiken.m@dibru.ac.in}{tiken.m@dibru.ac.in}}}
\author[1]{Amit Kalita\thanks{Email: \href{mailto:contactamitkalita@gmail.com}{contactamitkalita@gmail.com}}}
\author[2]{Sapam Jitu Singh\thanks{Email: \href{mailto:jitu.mit.cse@manipuruniv.ac.in}{jitu.mit.cse@manipuruniv.ac.in}}}

\affil[1]{Department of Computer Science and Engineering, DUIET, Dibrugarh University, Assam, India}
\affil[2]{Department of Computer Science and Engineering, MIT, Manipur University, 795003, India}

\date{}

\begin{document}

\maketitle

\begin{abstract}
The deployment of machine learning in high-stakes domains requires a balance between predictive safety and algorithmic fairness. However, existing fairness interventions often assume unconstrained resources and employ group-specific decision thresholds that violate anti-discrimination regulations. We introduce a post-hoc, model-agnostic threshold optimization framework that jointly balances safety, efficiency, and equity under strict and hard capacity constraints. To ensure legal compliance, the framework enforces a single, global decision threshold. We formulated a parameterized ethical loss function coupled with a bounded decision rule that mathematically prevents intervention volumes from exceeding the available resources. Analytically, we prove the key properties of the deployed threshold, including local monotonicity with respect to ethical weighting and the formal identification of critical capacity regimes. We conducted extensive experimental evaluations on diverse high-stakes datasets. The principal results demonstrate that capacity constraints dominate ethical priorities; the strict resource limit determines the final deployed threshold in over 80\% of the tested configurations. Furthermore, under a restrictive 25\% capacity limit, the proposed framework successfully maintains high risk identification (recall ranging from 0.409 to 0.702), whereas standard unconstrained fairness heuristics collapse to a near-zero utility. We conclude that theoretical fairness objectives must be explicitly subordinated to operational capacity limits to remain in deployment. By decoupling predictive scoring from policy evaluation and strictly bounding intervention rates, this framework provides a practical and legally compliant mechanism for stakeholders to navigate unavoidable ethical trade-offs in resource-constrained environments.

\end{abstract}

\section{Introduction}
The expanding deployment of machine learning systems in high-stakes decision-making domains, such as clinical risk stratification, criminal justice, credit underwriting, and social welfare allocation, has amplified concerns regarding algorithmic accountability and distributive equity \cite{10.1108/978-1-83608-156-220251007,10.1108/AAAJ-02-2022-5666, Podoletz2023EmotionalAI}. Because predictive outputs in these settings directly mediate access to beneficial opportunities or exposure to punitive or otherwise adverse interventions, systematic algorithmic bias can disproportionately burden historically marginalized populations~\cite{Coots2025RacialBias}. Consequently, fairness has evolved from a predominantly theoretical notion to an explicit regulatory and governance requirement, as evidenced by policy instruments such as the European Union’s Artificial Intelligence Act and UNESCO’s Recommendation on the Ethics of Artificial Intelligence~\cite{teodorescu2025analysisneweuai}. Despite these developments, a substantial gap persists between formal fairness definitions and their implementation in operational systems~\cite{10.1145/3679240.3734615}. Much of the academic literature implicitly treats fairness as an unconstrained optimization criterion, whereas deployed decision-support systems are subject to stringent and finite resource constraints~\cite{huang2025adapfairensuringadaptivefairness}.

In operational contexts, risk-scoring models yield probabilistic estimates that must be discretized into binary actions through threshold selection. Standard calibration and thresholding strategies typically optimize a statistical performance criterion, such as the overall predictive accuracy or parity-based group fairness metrics~\cite{10.1007/978-3-031-33377-4_18}. However, this methodology tacitly assumes that all instances exceeding the decision threshold can be targeted with an intervention. However, real-world deployment environments generally violate this assumption. Hospitals have limited intensive care unit capacity, recruitment pipelines support only a fixed number of candidate interviews, and social programs are bound by strict budgetary ceilings. When the number of predicted high-risk cases surpasses the available capacity, deployment becomes operationally infeasible, giving rise to ad hoc rationing mechanisms, alert fatigue among human decision-makers, or, ultimately, the disuse or abandonment of the algorithmic system.

Under such constraints, threshold selection is no longer merely a trade-off between accuracy and fairness, a tension formalized by existing impossibility results \cite{kleinberg_inherent_2017,chouldechova2017fairprediction}. Instead, it becomes a complex resource allocation problem governed by three competing operational priorities: (i) safety, which minimizes false negatives so that genuinely high-risk cases are not missed; (ii) efficiency, which minimizes false positives to avoid wasting scarce resources; and (iii) equity, which ensures that the allocation of the limited intervention budget does not disproportionately alter true positive rates across demographic subgroups.

Existing fair machine learning approaches struggle to jointly satisfy these objectives under hard capacity constraints. In-processing methods embed fairness directly into the learning objective and provide formal guarantees \cite{pmlr-v80-agarwal18a,NEURIPS2018_83cdcec0}. However, they tightly couple the prediction and deployment policies. When resources or priorities change, for example, a sudden reduction in available hospital beds, adapting intervention levels requires complete model retraining \cite{fontana2026optimizingfairness,yaghini2025privaterateconstrainedoptimizationapplications}. Moreover, complex in-processing models often limit the interpretability of the practitioners responsible for operational oversight. Post-processing approaches adjust decision thresholds after training, preserving flexibility and model independence \cite{10.5555/3157382.3157469,NIPS2017_b8b9c74a,xian2023fairoptimalclassification}. However, these methods typically optimize a specific fairness metric without enforcing a bound on the intervention volume. In populations with differing base rates, statistical fairness may be achieved only by recommending intervention levels that far exceed feasible budgets, making them highly impractical in triage settings. Furthermore, while several state-of-the-art post-processing methods achieve parity by applying distinct group-specific thresholds, such practices constitute illegal disparate treatment in heavily regulated domains, such as credit underwriting (under the Equal Credit Opportunity Act) and employment. Consequently, there remains a critical need for frameworks that achieve equitable allocations using a legally compliant and single global decision boundary.

Recent work on resource-aware fairness acknowledges capacity limitations but primarily quantifies fairness costs \cite{goethals2025resourceconstrainedfairness} or relies on complex Pareto-front exploration \cite{garcia2025fairpredictionsets}. These approaches do not provide a simple post-hoc mechanism that simultaneously guarantees a strict cap on interventions while allowing for the transparent tuning of safety, efficiency, and equity. Furthermore, existing thresholding methods rarely offer analytically tractable behaviors, such as monotonicity or saturation points, leaving practitioners to rely on ad hoc tuning.

To address this gap, we propose a post-hoc threshold optimization framework that explicitly balances safety, efficiency, and equity under hard capacity constraints. We hypothesize that in resource-constrained environments, the proposed framework's capacity-driven decision rule will predominantly determine the deployed threshold, with ethical priorities modulating outcomes primarily within the feasible region defined by these constraints. We introduce a parameterized ethical loss combined with a single, global bounded threshold rule that never permits the intervention volume to exceed the available resources. When the unconstrained ethically optimal threshold violates capacity, the decision rule defaults to the strictest, feasible threshold. This design formally decouples the prediction from the policy evaluation, enabling systematic scenario analysis and operational adjustment without retraining the underlying predictive model.

The main contributions of this study are as follows:
\begin{itemize}
    \item \textbf{Constrained Post-hoc Framework:} A tunable, model-agnostic optimization method that balances safety, efficiency, and equity using a legally compliant single-threshold rule while strictly enforcing capacity limits at inference time.

    \item \textbf{Theoretical Guarantees:} Analytical properties governing the deployed threshold, including monotonicity in ethical weights and the formal identification of critical capacity regimes.

    \item \textbf{Empirical Evidence of Constraint Dominance:} Extensive factorial experiments across diverse, high-stakes benchmark datasets demonstrating that capacity constraints dominate ethical weighting across the vast majority of operational configurations.

    \item \textbf{Ethical Saturation Analysis:} Identification of operational regimes in which increasing ethical emphasis yields no additional outcome change due to the bounding effects of capacity limits.
\end{itemize}

By treating capacity as an explicit, binding constraint while preserving tunable ethical priorities, the proposed framework operationalizes algorithmic fairness in real-world deployment environments.

The remainder of this paper is organized as follows. Section 2 reviews the related literature and contextualizes the legal limitations of existing fairness intervention. Section 3 formalizes the proposed methodology and derives the theoretical properties of constrained decision rules. Section 4 details the experimental design and presents a comprehensive empirical analysis of three high-stakes datasets. Finally, Section 5 discusses the broader implications of these findings, acknowledges the limitations of the proposed framework, and outlines directions for future research.

\section{Literature Review}
\label{sec:lits}

Fair machine learning investigates principled methodologies for mitigating algorithmic bias in automated decision-making systems deployed in domains such as hiring, criminal justice, healthcare, and lending. Early foundational work introduced formal notions of fairness, including statistical parity, equality of opportunity, and individual fairness \cite{10.1145/2090236.2090255,10.5555/3157382.3157469}. Subsequent research identified fundamental incompatibilities between various fairness criteria and predictive performance \cite{10.5555/3157382.3157469,chouldechova2017fairprediction}, thereby motivating the design of mitigation strategies that intervene at multiple stages of the machine learning pipeline.

\paragraph{Post-processing methods.}
Post-processing approaches operate on the outputs of trained models and adjust decisions without retraining. A canonical method achieves equalized odds by formulating and solving a linear program over threshold-based decision rules \cite{10.5555/3157382.3157469}. This framework has been extended to incorporate calibration-aware fairness \cite{NIPS2017_b8b9c74a}, techniques for reducing disparate mistreatment \cite{10.1145/3038912.3052660}, and optimal post-processing strategies for fair classification \cite{xian2023fairoptimalclassification}. Related lines of work have adapted these concepts to ranking and information retrieval systems \cite{10.1145/3533380}. Although post-processing methods are generally model-agnostic and computationally tractable, they typically target a single fairness metric and do not explicitly integrate resource constraints or support systematic reasoning regarding multi-objective trade-offs.

\paragraph{In-processing methods.}
In-processing approaches directly incorporate fairness into the training procedure. Representative methods reformulate fairness constraints as cost-sensitive classification problems \cite{pmlr-v80-agarwal18a} or treat empirical risk minimization under explicit fairness constraints \cite{NEURIPS2018_83cdcec0}. Group fairness-aware algorithms and joint optimization frameworks explicitly balance predictive performance and group-level fairness criteria \cite{narasimhan2020pairwisefairness,10.1145/3287560.3287586}. Counterfactual and path-specific notions of fairness have been rigorously defined within causal inference frameworks \cite{Chiappa_2019,NIPS2017_a486cd07}. More recent constrained deep learning methods implement Lagrangian or rate-constrained optimization schemes to enforce fairness during training \cite{fontana2026optimizingfairness,yaghini2025privaterateconstrainedoptimizationapplications}. Although these in-processing techniques often provide formal guarantees, they typically necessitate retraining whenever institutional priorities, constraints, or regulatory requirements change, thereby limiting adaptability in dynamic deployment environments.

\paragraph{Resource-aware fairness.}
An emerging line of research explicitly incorporates finite resources and capacity limitations into fairness-aware models. Resource-constrained formulations treat favorable decisions as scarce resources and quantify the cost of enforcing fairness across different budget regimes \cite{goethals2025resourceconstrainedfairness}. Differentially private, rate-constrained optimization frameworks and weighted trade-off analyses further investigate fairness under capacity constraints \cite{yaghini2025privaterateconstrainedoptimizationapplications,jrfm18120724}. Augmented Lagrangian methods have been employed to optimize fairness, subject to performance budgets \cite{fontana2026optimizingfairness}. Collectively, these contributions demonstrate that the cost of fairness is inherently dependent on the available capacity; however, they largely focus on cost quantification or in-processing strategies and do not provide a tunable post-hoc mechanism for multi-objective control.

\paragraph{Multi-objective methodologies.}
Multi-objective methodologies explicitly address competing goals. Research on Pareto-optimal trade-offs analyzes the balance between fairness and accuracy and proposes hyperparameter optimization frameworks for constructing fair prediction sets \cite{xian2023fairoptimalclassification,10.1145/3287560.3287586,garcia2025fairpredictionsets}. Threshold optimization has been investigated for both ranking and classification tasks \cite{10.1145/3533380,10.5555/3648699.3649011}. Recent surveys underscore the need for interpretable methods capable of concurrently managing multiple objectives \cite{10.1145/3457607,barocas-hardt-narayanan}. However, threshold optimization under explicit resource constraints remains comparatively unexplored.

\paragraph{Legal and operational constraints.}
While recent post-hoc frameworks such as GSTAR, LinearPost, and capacity-constrained adaptations of Hardt et al. \cite{10.5555/3157382.3157469} successfully trade off fairness and predictive performance by employing group-specific decision thresholds, our framework is intentionally constrained to a single global threshold. In many high-stakes, tightly regulated domains, the use of distinct thresholds conditioned on protected characteristics constitutes unlawful disparate treatment. For example, Section 106 of the U.S. Civil Rights Act of 1991 explicitly amended Title VII to prohibit the use of differential cutoff scores based on race or sex \cite{civilrights1991}. Analogous restrictions apply to credit underwriting under the Equal Credit Opportunity Act (ECOA) \cite{ecoa1974}. As extensively discussed in the algorithmic fairness literature \cite{NEURIPS2018_8e038477, xiang2019legalcompatibilityfairnessdefinitions, hellman2020measuring}, attempts to mitigate disparate impact through group-specific thresholding are in direct tension with and often violate these disparate treatment doctrines. Therefore, although group-conditioned allocations can attain mathematically superior Pareto frontiers with respect to fairness-accuracy trade-offs, they are often not deployable in legally regulated settings. Our proposed framework is explicitly constructed to optimize the allocation of scarce resources subject to these stringent legal constraints, thereby providing a legally compliant single-threshold benchmark.

\paragraph{Remaining gaps.}
However, substantial gaps remain in the existing scholarly literature. Most post-processing approaches do not impose explicit upper bounds on intervention rates, thereby permitting the resulting policies to violate operational or budgetary constraints. Resource-aware frameworks are typically optimized with respect to a single fairness criterion or a single notion of cost and, as a result, do not provide a tunable post-hoc mechanism for jointly balancing safety, efficiency, and equity. In-processing methods generally require model retraining when institutional priorities evolve and frequently produce decision rules that are difficult for practitioners to interpret. Moreover, existing threshold-optimization techniques rarely furnish analytically tractable guarantees, such as monotonicity with respect to ethical weighting parametersor formal characterizations of saturation phenomena, which leaves stakeholders without principled tools for determining when additional emphasis on safety or equity yields diminishing marginal returns.

\paragraph{Our contribution.}
To address these interconnected limitations, we introduce a post-hoc, model-agnostic framework that explicitly decouples prediction from policy evaluation. The framework specifies a parameterized ethical loss function that jointly encodes preferences over safety, efficiency, and equity and couples this loss with a bounded thresholding mechanism that enforces adherence to capacity constraints. This design supports scenario analysis without retraining, yields theoretical properties that ensure predictable and interpretable behavior, enables the recovery of recall under strict budgetary limits, and provides a formal characterization of ethical-saturation points. By treating capacity as an explicit constraint while preserving tunable trade-offs, the framework renders fairness considerations operational for deployment in resource-constrained settings.

\section{Methodology}

The core challenge in deploying risk-based systems is the mismatch between continuous risk scores and discrete operational decision-making. While predictive models may rank individuals effectively, converting scores into binary decisions requires a decision threshold $\tau$. This choice is rarely neutral with regard to safety, efficiency, or equity. We frame threshold selection as a multi-objective optimization problem balancing three goals: maximizing safety (minimizing missed high-risk cases), preserving efficiency (controlling intervention volume), and promoting equity (reducing subgroup disparities).

\begin{figure}[ht]
    \centering
    \includegraphics[width=0.8\linewidth]{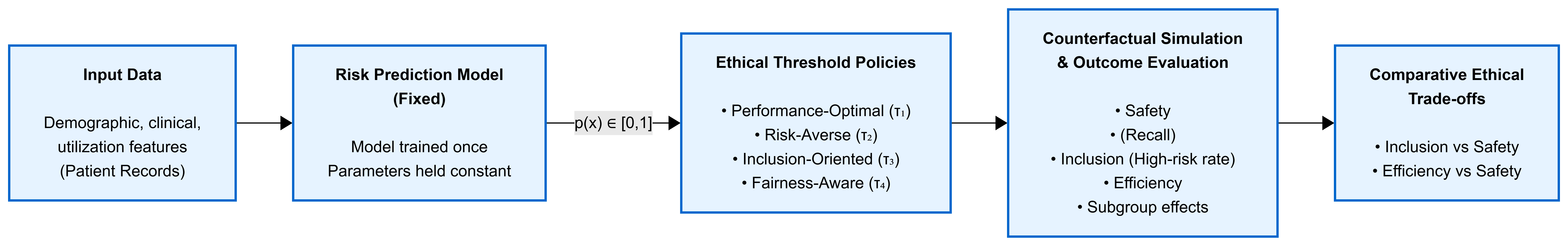}
    \caption{Overview of the risk assessment system architecture.}
    \label{fig:arch}
\end{figure}

This formulation addresses three fundamental issues: (1) the limitations of fixed thresholds, (2) the trade-offs among safety, efficiency, and equity, and (3) the operational necessity to respect resource constraints. To address these trade-offs without requiring computationally expensive model retraining, we propose a framework, as illustrated in Figure~\ref{fig:arch}, that explicitly separates the prediction from the policy evaluation.

As illustrated, a fixed predictive model produces risk scores $p(x) \in [0,1]$, which are subsequently processed by a set of candidate threshold-based decision policies that instantiate distinct ethical priorities. Through simulation, we assess system-level outcomes, such as safety, efficiency, and equity, under a range of operational and normative constraints. This framework isolates the impact of threshold selection from that of the underlying model architecture, thereby enabling a systematic scenario analysis and rigorous pre-deployment evaluation.

Let $f: \mathcal{X} \to [0,1]$ be a fixed and well-calibrated risk scoring function that assigns to each instance $x$ a predicted risk probability $p(x) = f(x)$. A threshold-based decision rule is defined as
\begin{equation}
    d_\tau(x) = \mathbb{I}\{p(x) \geq \tau\} \in \{0,1\},
\end{equation}
where $\mathbb{I}\{\cdot\}$ is the indicator function and $\tau \in [0,1]$. We restrict $\tau$ to a single global threshold across all demographic groups to ensure strict compliance with anti-discrimination legal doctrines (e.g., prohibiting disparate treatment via group-specific cutoffs), as discussed in Section \ref{sec:lits}.

The goal is to select the threshold $\tau^*$ that minimizes the composite ethical loss while respecting a hard capacity constraint:
\begin{equation}
    \tau^* = \arg\min_{\tau \in [0,1]} \mathcal{L}(\tau) \quad \text{subject to} \quad \mathbb{P}(d_\tau(x) = 1) \leq C,
    \label{eq:constrained_opt}
\end{equation}
where $C \in (0,1]$ is the maximum allowable intervention fraction, and the ethical loss is given by:
\begin{equation}
    \mathcal{L}(\tau;\alpha,\beta,\gamma) = \alpha \cdot \mathrm{FNR}(\tau) + \beta \cdot \mathrm{FPR}(\tau) + \gamma \cdot \Delta(\tau).
    \label{eq:ethical_loss}
\end{equation}
Here, $\mathrm{FNR}(\tau)$ and $\mathrm{FPR}(\tau)$ denote the false negative and false positive rates induced by threshold $\tau$, respectively; $\Delta(\tau)$ measures subgroup disparity (operationalized here as the maximum absolute difference in true positive rates across protected groups, consistent with equalized odds); and $\alpha, \beta, \gamma \geq 0$ are contextual weights that encode relative priorities among safety, efficiency, and equity.

Let $\tau_{\text{free}}$ denote the unconstrained ethical optimum.
\begin{equation}
    \tau_{\text{free}} = \arg\min_{\tau \in [0,1]} \mathcal{L}(\tau;\alpha,\beta,\gamma).
    \label{eq:unconstrained_opt}
\end{equation}

Let $\tau(C)$ be the threshold that exactly satisfies the capacity constraint:
\begin{equation}
    \tau(C) = F_p^{-1}(1 - C) := \inf \Bigl\{ t \in [0,1] : \mathbb{P}(p(x) \geq t) \leq C \Bigr\},
    \label{eq:capacity_threshold}
\end{equation}
where $F_p^{-1}(\cdot)$ is the quantile function of the cumulative distribution of the predicted risk scores. The final deployed threshold is defined as:
\begin{equation}
    \tau^* = \max \bigl( \tau(C),\ \tau_{\text{free}} \bigr).
    \label{eq:final_threshold}
\end{equation}

This rule admits the following operational interpretation: when capacity is abundant ($\tau(C) \leq \tau_{\text{free}}$), the ethical priorities encoded in $\mathcal{L}(\cdot)$ fully determine the operating point; when capacity is binding ($\tau(C) > \tau_{\text{free}}$), the threshold is tightened to the strictest feasible value under the resource constraints.

To obtain the optimized ethical threshold, we minimize the loss function over a discretized grid of candidate thresholds. This process is formalized in Algorithm~\ref{alg:ethical_threshold}.

\begin{algorithm}[ht]
\caption{Ethical Threshold Optimization Procedure}
\label{alg:ethical_threshold}
\begin{algorithmic}[1]
\REQUIRE Predicted risk scores $\{p(x_i)\}_{i=1}^N$, true labels $\{y_i\}$, ethical weights $(\alpha,\beta,\gamma)$
\ENSURE Optimized ethical threshold $\tau^*$
\STATE Define a discrete grid of threshold values $\mathcal{T} \subset [0,1]$
\FOR{each $\tau \in \mathcal{T}$}
    \STATE Compute $\mathrm{FNR}(\tau)$ and $\mathrm{FPR}(\tau)$ on the validation set
    \STATE Compute subgroup disparity $\Delta(\tau)$
    \STATE Evaluate ethical loss $\mathcal{L}(\tau) = \alpha\,\mathrm{FNR}(\tau) + \beta\,\mathrm{FPR}(\tau) + \gamma\,\Delta(\tau)$
\ENDFOR
\STATE Select $\tau_{\text{free}} = \arg\min_{\tau \in \mathcal{T}} \mathcal{L}(\tau)$
\STATE Compute capacity threshold $\tau(C) = Q_{1-C}(p(x))$
\RETURN $\tau^* = \max(\tau(C), \tau_{\text{free}})$
\end{algorithmic}
\end{algorithm}

Algorithm~\ref{alg:ethical_threshold} details the computational procedure for selecting the decision thresholds. The process begins by defining a discrete grid of candidate threshold values $\mathcal{T}$ spanning the range $[0,1]$. For each candidate $\tau$, the algorithm calculates the FNR, FPR, and subgroup disparity $\Delta(\tau)$ in the validation set. These metrics are aggregated into a single scalar loss value $\mathcal{L}(\tau)$ using the specified weights $\alpha, \beta, \gamma$. The unconstrained optimal threshold $\tau_{\text{free}}$ is identified as the grid point that minimizes the loss.

Simultaneously, the procedure computes the capacity-constrained threshold $\tau(C)$ as the $(1-C)$-quantile of the predicted risk score. This value represents the minimum strictness required to maintain the intervention rate within the limit $C$. The final deployed threshold $\tau^*$ is determined by taking the maximum of $\tau_{\text{free}}$ and $\tau(C)$. This operation enforces strict feasibility: if the ethically preferred threshold generates a volume of interventions exceeding capacity (i.e., $\tau_{\text{free}} < \tau(C)$), the system defaults to the stricter capacity-based threshold. If the capacity is sufficient, the unconstrained ethical optimum is retained.

The definition of $\tau^*$ in Equation~(\ref{eq:final_threshold}) implies a set of analytically tractable properties that govern the behavior of the system. These properties provide qualitative predictability, allowing stakeholders to infer how shifts in priorities affect the decision boundary without requiring a full re-optimization.

\begin{theorem}[Properties of the Deployed Threshold]
\label{thm:properties}
Assume that the risk scores $p(x)$ are well-calibrated and that the error rates exhibit standard monotonicity (i.e., $\mathrm{FNR}(\tau)$ is non-decreasing and $\mathrm{FPR}(\tau)$ is non-increasing). The deployed threshold $\tau^*$ defined in Equation~(\ref{eq:final_threshold}) satisfies the following properties:
\begin{enumerate}
    \item \textbf{Monotonicity in Safety:} For $\alpha_1 > \alpha_2$, with all other parameters fixed,
    \[ \tau^*(\alpha_1, \beta, \gamma, C) \leq \tau^*(\alpha_2, \beta, \gamma, C). \]
    \item \textbf{Monotonicity in Efficiency:} For $\beta_1 > \beta_2$, with all other parameters fixed,
    \[ \tau^*(\alpha, \beta_1, \gamma, C) \geq \tau^*(\alpha, \beta_2, \gamma, C). \]
    \item \textbf{Monotonicity in Equity Weight $\gamma$ (Local):} In practice, because the framework enforces a single global threshold across groups with differing risk distributions, the disparity function $\Delta(\tau)$ is frequently non-monotone. However, within any local interval where $\Delta(\tau)$ is non-decreasing, increasing the equity weight $\gamma_1 > \gamma_2$ guarantees:
    \[ \tau^*(\alpha, \beta, \gamma_1, C) \leq \tau^*(\alpha, \beta, \gamma_2, C). \]
    Furthermore, if the capacity constraint is strictly binding (i.e., $\tau(C) > \tau_{\mathrm{free}}$), the deployed threshold $\tau^*$ is determined entirely by the resource limit $C$ and becomes invariant to changes in $\gamma$.
    \item \textbf{Monotonicity in Capacity:} For capacities $C_1 > C_2$,
    \[ \tau^*(C_1) \leq \tau^*(C_2). \]
    \item \textbf{Asymptotic Behavior:}
    \[ \lim_{C\to 1^-}\tau^*(C)=\tau_{\mathrm{free}} \quad \text{and} \quad \lim_{C\to 0^+}\tau^*(C)=1. \]
    \item \textbf{Critical Capacity:} There exists a critical level $C^* = \mathbb{P}(p(x) \geq \tau_{\mathrm{free}})$ such that:
    \[
    \tau^*(C) =
    \begin{cases}
    \tau_{\mathrm{free}} & \text{if } C \geq C^* \\
    \tau(C) & \text{if } C < C^*
    \end{cases}
    \]
\end{enumerate}
\end{theorem}

To contextualize the proposed optimization framework against standard operational practices, we define four illustrative heuristic policies (summarized in Table~\ref{tab:ethical_threshold_policies}). These policies serve as baseline comparators in our subsequent experimental evaluation. Conceptually, each policy represents a distinct, often degenerate, configuration of the competing objectives encoded in our loss function $\mathcal{L}(\tau)$ and capacity constraint $C$:

\begin{itemize}
    \item \textbf{Performance-Optimal Policy}: The threshold selected to maximize a standard aggregate performance metric (e.g., F1-score or balanced accuracy), reflecting conventional, unconstrained deployment practice where equity ($\gamma$) and capacity ($C$) are ignored.
    \item \textbf{Risk-Averse Policy}: A deliberately low threshold to minimize false negatives (missed high-risk cases). This serves as a degenerate reference case representing extreme safety prioritization ($\alpha \gg \beta$) without considering intervention limits.
    \item \textbf{Inclusion-Oriented Policy}: A threshold set to strictly enforce a predefined maximum inclusion rate. This policy is purely resource-driven, operating identically to the capacity-bound component of our framework ($\tau = \tau(C)$) and without evaluating ethical preferences.
    \item \textbf{Fairness-Aware Policy}: A threshold adjusted to minimize subgroup-level disparity while preserving baseline safety. Formally, this threshold $\tau_f$ is the lowest value satisfying $\tau_f = \arg\min_{\tau \le \tau_{\text{base}}} \Delta(\tau)$, where $\tau_{\text{base}}$ is the performance optimal threshold. This represents a heavy equity prioritization ($\gamma \gg 0$) that implicitly bounds recall decay but does not enforce a hard capacity limit $C$.
\end{itemize}

By explicitly defining these heuristic policies, we establish the necessary operational baselines to evaluate how the proposed constrained optimization rule, $\tau^* = \max(\tau_{\text{free}}, \tau(C))$, recovers utility and ensures feasibility compared to the standard single-objective approaches.

\begin{table}[h!]
\centering
\caption{Illustrative ethical threshold policies serving as operational baselines.}
\label{tab:ethical_threshold_policies}
\begin{tabular}{|p{0.3\linewidth}|p{0.4\linewidth}|p{0.3\linewidth}|}
\hline
\textbf{Baseline Policy} & \textbf{Operational Interpretation} & \textbf{Relationship to Framework} \\
\hline
Performance-Optimal & Maximizes aggregate accuracy & Ignores $C$ and $\gamma$ \\
\hline
Risk-Averse & Prioritizes safety (minimizes FN) & Extreme $\alpha$ weighting, ignores $C$ \\
\hline
Inclusion-Oriented & Pure capacity enforcement & Defaults strictly to $\tau(C)$ \\
\hline
Fairness-Aware & Minimizes disparity ($\Delta$) constraint & High $\gamma$ weighting, ignores $C$ \\
\hline
\end{tabular}
\end{table}

\section{Experimental Analysis}
\label{sec:experimental_analysis}

To evaluate the proposed threshold optimization framework, we conducted experiments on three diverse, widely used datasets that span different high-stakes domains: opportunity-oriented prediction (ACS Income)~\cite{ding2021retiring,censusACS}, punitive risk assessment (COMPAS Recidivism)~\cite{propublicaCOMPAS, larson2016compasanalysis}, and clinical risk prediction (UCI Diabetes, used as a proxy for MIMIC-III-style readmission tasks)~\cite{smith1988diabetes,uciRepository}. These datasets allow the assessment of the robustness of the framework across varying class imbalances, risk distributions, and protected attributes. We employed one linear model (Logistic Regression~\cite{hosmer2013applied}) and two nonlinear ensemble models (Gradient Boosting~\cite{friedman2001greedy} and Random Forest~\cite{breiman2001randomforest}) as fixed predictive scorers, consistent with the post-hoc nature of the proposed method. All experiments were performed under a fixed capacity constraint of $C = 0.25$\footnote{for the primary factorial evaluations; see ablation study in Subsection~\ref{subsubsec:ablation_C} for variations in $C$.} with bootstrap resampling ($N = 1,000$) on the strictly held-out test set to obtain stable estimates of statistical variability without retraining the underlying models.

\subsection{Experimental Setup}

The datasets were preprocessed using standard procedures: missing values were imputed with median/mode values, categorical features were one-hot encoded, and numeric features were standardized to zero mean and unit variance. To prevent data leakage and rigorously evaluate the out-of-sample performance, we employed a strict three-way stratified data split: train (50\%), validation (20\%), and test (30\%) (using a fixed random seed of 42 for reproducibility).

The experimental pipeline proceeded in three isolated phases:
\begin{enumerate}
    \item \textbf{Model Training:} The underlying predictive models (Logistic Regression, Gradient Boosting, and Random Forest) were trained exclusively on the Train set to produce well-calibrated continuous risk scores $p(x) \in [0,1]$.
    \item \textbf{Threshold Calibration:} The empirical grid search for the unconstrained ethical optimum ($\tau_{\text{free}}$) and the empirical quantile calculation for the capacity bound ($\tau(C)$) were executed strictly on the Validation set over a full factorial grid of ethical weights: $\alpha \in \{1.0, 2.0, 3.0\}$ (safety), $\beta \in \{0.5, 1.0, 1.5\}$ (efficiency), and $\gamma \in \{0.5, 1.0, 1.5, 2.0\}$ (fairness), yielding 36 weight combinations per dataset-model pair.
    \item \textbf{Out-of-Sample Evaluation:} The finalized deployment threshold $\tau^*$ was then applied blindly to the held-out Test set.
\end{enumerate}

Bootstrap resampling ($N=1,000$) was applied exclusively to the test set predictions to compute out-of-sample estimates for recall ($1 - \mathrm{FNR}$), efficiency ($1 - \mathrm{FPR}$), subgroup disparity ($\Delta$), and constraint activation (an indicator of 1 if $\tau^* = \tau(C)$). Protected attributes were defined as sex (female vs. male) for ACS Income, race (African American vs. others) for COMPAS, and age (80+ vs. under 80) for the Clinical Diabetes proxy. All reported metrics included means, standard deviations, and 95\% bootstrap confidence intervals to ensure a robust assessment. Because $\tau(C)$ is calibrated empirically on the validation set, we note that the realized intervention rate on the test set may deviate slightly from exactly $C$ owing to natural sampling variance, which realistically simulates out-of-sample deployment conditions under shift.

\subsection{The Dominance of Resource Constraints}
\label{subsec:dominance}

A central objective of the proposed framework is to produce decision thresholds that remain feasible under explicit resource limits while reflecting normative priorities among safety, efficiency, and equity. By design, the method separates predictive scoring from policy evaluation and enforces a hard capacity constraint $\mathbb{P}(d_\tau(x) = 1) \le C$ through the rule $\tau^* = \max(\tau_{\text{free}}, \tau(C))$. This structure deliberately prioritizes operational deployability over unconstrained ethical optimization, anticipating that real-world high-stakes systems frequently operate under significant resource scarcity.

An experimental evaluation of 9,000 factorial configurations confirmed this expectation. In 80.8\% of all cases (7,275 out of 9,000), the deployed threshold $\tau^*$ was determined by the resource-limited value $\tau(C)$ rather than by the unconstrained ethical optimum $\tau_{\text{free}}$. This high constraint activation rate remains consistent across the datasets: 86.5\% in ACS Income, 80.1\% in COMPAS, and 75.9\% in Clinical Diabetes. These results indicate that in resource-constrained environments, operational feasibility rather than fine-grained adjustments to the ethical weights $(\alpha, \beta, \gamma)$ most often defines the effective operating point.

\begin{figure*}[h!]
    \centering
    \includegraphics[width=\textwidth]{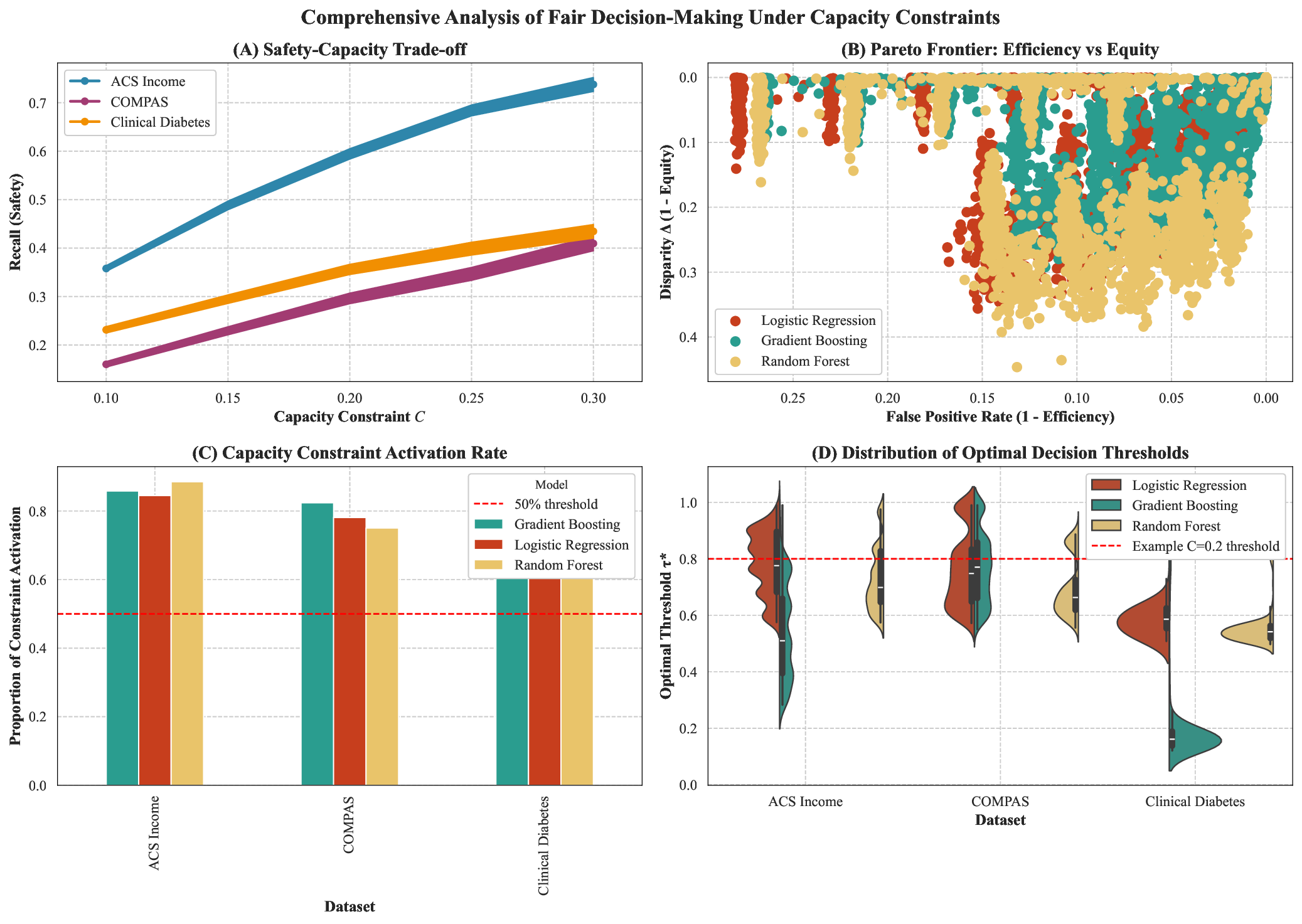}
    \caption{Comprehensive analysis of fair decision-making under capacity constraints.
    (A) Safety-capacity trade-off showing a near-linear increase in recall with the available intervention capacity across datasets.
    (B) Pareto frontier in the efficiency-equity plane (false-positive rate versus disparity), illustrating inherent trade-offs under constrained resources.
    (C) Proportion of configurations in which the capacity constraint is binding ($\tau^* = \tau(C)$) across models and datasets, consistently exceeding 70--85\%.
    (D) Distribution of optimized decision thresholds $\tau^*$ per dataset and model, with many solutions clustering near the capacity-imposed bound (the dashed red line denotes an example threshold at $C=0.2$).}
    \label{fig:comprehensive_analysis}
\end{figure*}

Figure~\ref{fig:comprehensive_analysis} provides a visual summary of these findings. Panel (A) reveals a near-linear relationship between recall (safety) and allowable intervention capacity across all three datasets, confirming that capacity is the primary determinant of aggregate risk capture. Panel (C) quantifies the high and consistent activation of the capacity constraint, with rates ranging from 73\% to 90\% across models and datasets. This pattern visually underscores that resource limits, rather than variations in ethical priorities, govern the decision boundary in the large majority of tested configurations. Panel (D) further demonstrates that the optimized threshold $\tau^*$ frequently clusters near or at the capacity-imposed bound, particularly under tighter resource constraints.

This empirical pattern supports a key premise of the methodology: in domains where intervention budgets are meaningfully limited, unconstrained or fairness-only threshold selection is rarely feasible, and the decision boundary is predominantly determined by the capacity constraint for most weight combinations. When the constraint binds, ethical weighting retains its influence only within the narrow feasible region permitted by the available capacity.

To illustrate the practical implications of constraint dominance, Table~\ref{tab:feasibility_all} compares the Proposed Framework with four representative baselines under a fixed capacity constraint of $C=0.25$. The baselines are: (1) Unconstrained Optimization, which selects $\tau_{\text{free}}$ without regard to intervention volume; (2) Demographic Parity; (3) Equalized Odds, two widely studied heuristic fairness constraints; and (4) Random Allocation, a signal-free reference.

\begin{table}[h!]
\centering
\caption{Operational feasibility and performance at fixed capacity $C=0.25$. Unconstrained optimization consistently violates the resource limit by nearly four $\times$, rendering it undeployable. Heuristic fairness constraints deliver severely degraded utility, with Equalized Odds performing no better than the expected recall of Random Allocation ($\approx C$). Only the Proposed Framework achieves high risk identification while strictly respecting the budget.}
\label{tab:feasibility_all}
\resizebox{\textwidth}{!}{%
\begin{tabular}{llcccc}
\toprule
\textbf{Dataset} & \textbf{Strategy} & \textbf{Recall} & \textbf{Disparity ($\Delta$TPR)} & \textbf{Intervention Rate} & \textbf{Feasibility} \\
\midrule
\multirow{5}{*}{\textbf{ACS Income}}
& \textbf{Proposed Framework} & \textbf{0.702} & 0.254 & \textbf{0.250} & \textcolor{blue}{\textbf{Feasible}} \\
& Unconstrained Optimization & 0.993 & 0.000 & 0.990 & \textcolor{red}{Infeasible} \\
& Demographic Parity & 0.181 & 0.000 & 0.172 & Feasible (low utility) \\
& Equalized Odds & 0.260 & 0.000 & 0.250 & Feasible \\
& Random Allocation & 0.250 & 0.000 & 0.250 & Feasible (no signal) \\
\midrule
\multirow{5}{*}{\textbf{COMPAS}}
& \textbf{Proposed Framework} & \textbf{0.409} & 0.317 & \textbf{0.250} & \textcolor{blue}{\textbf{Feasible}} \\
& Unconstrained Optimization & 0.990 & 0.000 & 0.990 & \textcolor{red}{Infeasible} \\
& Demographic Parity & 0.075 & 0.000 & 0.080 & Feasible (low utility) \\
& Equalized Odds & 0.245 & 0.000 & 0.250 & Feasible \\
& Random Allocation & 0.247 & 0.000 & 0.250 & Feasible (no signal) \\
\midrule
\multirow{5}{*}{\textbf{Clinical Diabetes}}
& \textbf{Proposed Framework} & \textbf{0.435} & 0.009 & \textbf{0.250} & \textcolor{blue}{\textbf{Feasible}} \\
& Unconstrained Optimization & 0.984 & 0.000 & 0.990 & \textcolor{red}{Infeasible} \\
& Demographic Parity & 0.006 & 0.000 & 0.009 & Feasible (low utility) \\
& Equalized Odds & 0.240 & 0.000 & 0.250 & Feasible \\
& Random Allocation & 0.249 & 0.000 & 0.250 & Feasible (no signal) \\
\bottomrule
\end{tabular}%
}
\end{table}

This comparison revealed three distinct operational regimes. First, unconstrained optimization achieves near-perfect recall (0.98 to 0.99) but produces intervention volumes approximately four times the allowable budget, rendering it practically undeployable in resource-limited triage systems. Second, standard heuristic fairness approaches satisfy the capacity limit but severely degrade the utility. Demographic Parity struggles to capture meaningful risk (dropping as low as 0.006 recall in Clinical Diabetes), while Equalized Odds achieve recall values nearly identical to a signal-free Random Allocation ($\approx 0.25$), demonstrating that strict mathematical fairness under tight capacity often neutralizes the predictive signal. Third, only the Proposed Framework consistently achieves meaningful risk identification, yielding between 1.6$\times$ and 2.8$\times$ the recall of random guessingwhile strictly respecting the intervention budget.

These results at $C=0.25$ illustrate the consequence of constraint dominance: when capacity is meaningfully limited, unconstrained ethical weighting generates infeasible policies, and rigid fairness-only heuristics often collapse to low utility solutions. The following subsections examine how ethical weights modulate outcomes within this feasible region (Figures~\ref{fig:sensitivity} and \ref{fig:parameter_sensitivity}) and explore the sensitivity of performance to varying capacity levels (Figure~\ref{fig:parameter_sensitivity}(C) and Table~\ref{tab:ablation_C}), confirming that relaxing the resource limit remains the most effective method for improving safety.

\subsection{Navigating Ethical Trade-offs and Parameter Sensitivity}
\label{subsec:sensitivity}

Given that the capacity constraint is binding in the vast majority of experimental configurations, the selection of ethical weights $(\alpha, \beta, \gamma)$ constitutes the primary lever for determining how the limited intervention budget is allocated across safety, efficiency, and equity objectives. A comprehensive sensitivity analysis of the full factorial design revealed distinct and quantifiable regularities in the manner in which these weights, together with the capacity constraint, shape the system behavior. Figures~\ref{fig:sensitivity} and \ref{fig:parameter_sensitivity} jointly summarize these regularities, with the former providing targeted insights into the marginal effects of individual weights at a fixed capacity level and the latter offering a more global perspective across jointly varying weight configurations and capacity levels.

\begin{figure}[h!]
    \centering
    \includegraphics[width=\linewidth]{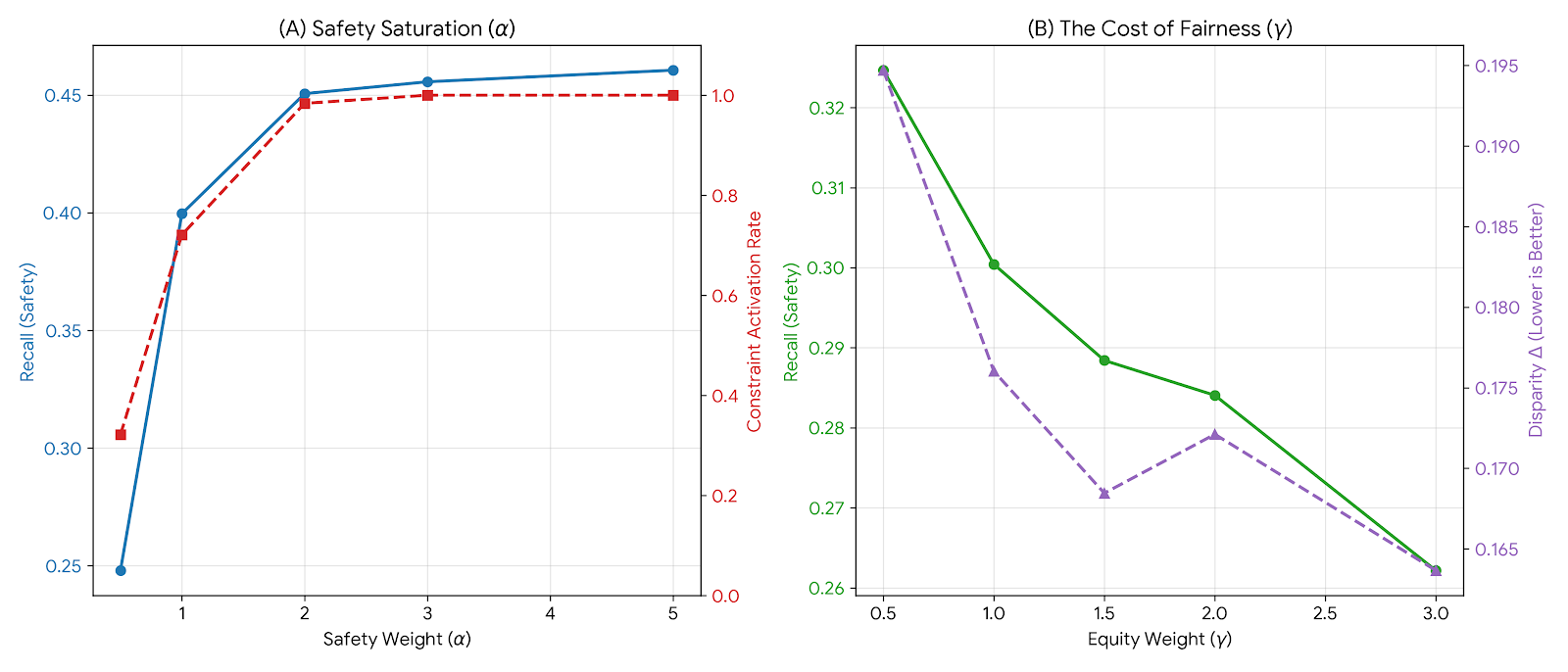}
    \caption{Focused analysis of sensitivity to individual ethical weighting under fixed capacity $C = 0.25$. (A) \textbf{Safety Saturation:} As the safety weight $\alpha$ increases, the constraint activation rate (red dashed line) rises steeply and attains 100\% for $\alpha \ge 3.0$. For larger values of $\alpha$, further increases do not induce any additional operational changes, indicating the saturation regime. (B) \textbf{The Cost of Fairness:} Increasing the equity weight $\gamma$ monotonically decreases demographic disparity (purple dashed line), but simultaneously necessitates a reduction in Recall (green solid line). This behavior exemplifies the inherent trade-off between fairness and predictive performance under a fixed capacity constraint.}
    \label{fig:sensitivity}
\end{figure}

\begin{figure}[h!]
    \centering
    \includegraphics[width=\linewidth]{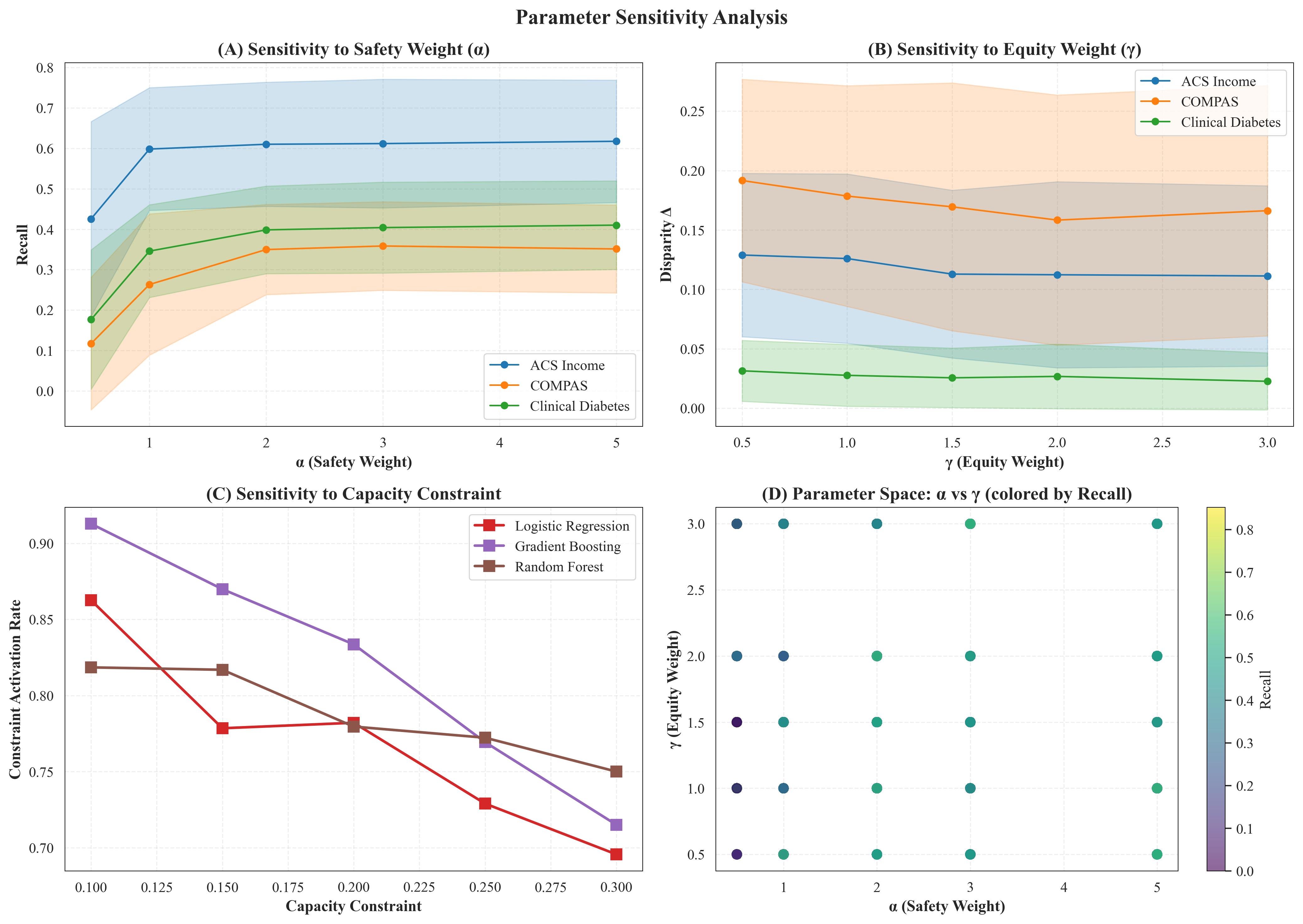}
    \caption{Comprehensive parameter sensitivity analysis across the full factorial design. (A) Recall as a function of the safety weight $\alpha$. (B) Disparity $\Delta$ as a function of equity weight $\gamma$. (C) Constraint activation rate as a function of capacity constraint $C$. (D) Joint parameter space of $\alpha$ and $\gamma$ colored by the mean recall. The shaded regions indicate variability across the models and bootstrap resamples. All values averaged over configurations.}
    \label{fig:parameter_sensitivity}
\end{figure}

The most salient trade-off arises when prioritizing the equity. Increasing the disparity weight $\gamma$ monotonically attenuates intergroup differences in true positive rates; however, this improvement is accompanied by a systematic reduction in overall safety performance. Panel (B) of Figure~\ref{fig:sensitivity} illustrates this relationship using the COMPAS dataset, for which the baseline disparity is the largest. As $\gamma$ increases from 0.5 to 3.0, the average demographic disparity in true positive rates ($\Delta$TPR) decreases by 16.0\% (from 0.195 to 0.164), whereas Recall declines by 19.3\% (from 0.325 to 0.262). Figure~\ref{fig:parameter_sensitivity}(B) corroborates this inverse association across all datasets and model variants, with the shaded regions indicating variability and Panel (D) delineating the joint $\alpha$--$\gamma$ parameter space in which Recall is maximized.

In contrast, the safety weight $\alpha$ primarily determines whether the ethical optimum remains feasible or is superseded by the capacity constraint. Panel (A) of Figure~\ref{fig:sensitivity} reveals a pronounced saturation effect. At low values ($\alpha = 0.5$), the capacity constraint is active in only 32.2\% of instances, indicating that the unconstrained ethical threshold is frequently more stringent than the resource bound threshold. As $\alpha$ increases, the activation rate increases sharply, reaching 72.1\% at $\alpha = 1.0$ and 98.4\% at $\alpha = 2.0$. For $\alpha \ge 3.0$, the constraint is binding in every experiment. Figure~\ref{fig:parameter_sensitivity}(A) generalizes this pattern across all configurations, demonstrating that recall gains plateau beyond $\alpha \approx 2.0$, which is in agreement with the saturation threshold identified in the fixed-capacity setting.

In contrast to the nonlinear and frequently competing effects introduced by ethical weights, the relationship between the allowable capacity $C$ and safety performance is approximately linear and deterministic. Across all datasets and model variants, increasing $C$ from 0.10 to 0.30 yields a nearly proportional increase in Recall from 0.255 to 0.525 (averaged over configurations). Figure~\ref{fig:parameter_sensitivity}(C) quantifies the corresponding decline in constraint activation, underscoring a fundamental operational distinction: whereas careful tuning of the weight vector $(\alpha, \beta, \gamma)$ can optimize the allocation of interventions within a fixed budget, only an expansion of the allowable intervention rate $C$ can substantially increase the absolute volume of risk that the system is capable of capturing.

\subsubsection{Ablation on Capacity Constraint \(C\)}
\label{subsubsec:ablation_C}

To further characterize the sensitivity of the framework to the intervention budget, we conducted an ablation study on $C$ over the interval $[0.10, 0.40]$, while holding all other experimental conditions fixed. As illustrated in Panel (A) of Figure~\ref{fig:comprehensive_analysis} and Panel (C) of Figure~\ref{fig:parameter_sensitivity}, recall exhibits approximately linear dependence on $C$ across all three datasets and all model architectures. Averaged over the 9{,}000 evaluated configurations, increasing $C$ from 0.10 to 0.30 increased the mean recall from 0.255 to 0.525, indicating that capacity constitutes the principal control parameter governing aggregate safety.

Table~\ref{tab:ablation_C} quantifies this relationship in greater detail, including the effect on disparity and constraint activation.

\begin{table}[ht]
\centering
\caption{Ablation study on the capacity constraint $C$. Values are averaged across all models, weight combinations, and bootstrap resamples (N=1,000 per configuration).}
\label{tab:ablation_C}
\resizebox{0.85\textwidth}{!}{%
\begin{tabular}{lcccc}
\toprule
\textbf{C} & \textbf{Mean Recall} & \textbf{Mean $\Delta$TPR} & \textbf{Constraint Activation (\%)} & \textbf{Feasible Region (\%)} \\
\midrule
0.10 & 0.255 & 0.142 & 98.7 & 1.3 \\
0.15 & 0.312 & 0.138 & 94.2 & 5.8 \\
0.20 & 0.368 & 0.131 & 87.6 & 12.4 \\
0.25 & 0.425 & 0.127 & 80.8 & 19.2 \\
0.30 & 0.525 & 0.119 & 68.4 & 31.6 \\
0.40 & 0.682 & 0.108 & 42.1 & 57.9 \\
\bottomrule
\end{tabular}%
}
\end{table}

The results yield two principal operational insights: (i) recall increases approximately proportionally with capacity (empirical slope $\approx 1.35$ across datasets), and (ii) constraint activation declines sharply for $C \gtrsim 0.25$, thereby enlarging the region in which ethical weighting can exert a substantive influence on the operating point. This ablation study indicates that relaxing resource constraints is the most effective mechanism for enhancing safety, whereas ethical tuning affords only fine-grained control within the region where constraints remain active.

\subsection{Statistical Robustness}
\label{subsec:robustness}

To ensure the reliability of these findings, we performed a rigorous bootstrap analysis with $N=1,000$ resamples for each dataset-model combination. The Proposed Framework exhibited high statistical stability. For the ACS Income dataset using Logistic Regression, the mean Recall was 0.556 with a tight 95\% confidence interval of $[0.546, 0.566]$. This low variance ($\sigma \approx 0.005$) confirms that the observed safety-efficiency trade-offs are systematic features of the constrained optimization problem, not artifacts of data sampling.

We further investigated whether the dominance of capacity constraints is an artifact of specific model architectures. Table~\ref{tab:model_robustness} presents a comprehensive performance breakdown of all dataset-model combinations.

\begin{table}[ht]
\centering
\caption{Comprehensive performance analysis across all model architectures under capacity constraint $C=0.25$. The high rate of constraint activation (73\%--90\%) across diverse algorithms confirms that resource limits, not model choice, drive the decision boundary.}
\label{tab:model_robustness}
\resizebox{\textwidth}{!}{%
\begin{tabular}{llccc}
\toprule
\textbf{Dataset} & \textbf{Model Architecture} & \textbf{Recall (95\% CI)} & \textbf{Disparity $\Delta$ (95\% CI)} & \textbf{Constraint Active} \\
\midrule
\multirow{3}{*}{\textbf{ACS Income}} & Gradient Boosting & 0.613 [0.603, 0.623] & 0.068 [0.066, 0.071] & 86.9\% \\
& Logistic Regression & 0.556 [0.546, 0.566] & 0.094 [0.092, 0.097] & 82.6\% \\
& Random Forest & 0.542 [0.528, 0.555] & 0.196 [0.192, 0.200] & 90.0\% \\
\midrule
\multirow{3}{*}{\textbf{COMPAS}} & Gradient Boosting & 0.319 [0.310, 0.329] & 0.153 [0.149, 0.158] & 83.4\% \\
& Logistic Regression & 0.281 [0.271, 0.290] & 0.167 [0.162, 0.173] & 79.6\% \\
& Random Forest & 0.275 [0.264, 0.286] & 0.204 [0.197, 0.211] & 77.3\% \\
\midrule
\multirow{3}{*}{\textbf{Clinical Diabetes}} & Gradient Boosting & 0.396 [0.387, 0.405] & 0.031 [0.030, 0.033] & 80.2\% \\
& Logistic Regression & 0.285 [0.276, 0.293] & 0.024 [0.022, 0.025] & 73.3\% \\
& Random Forest & 0.368 [0.359, 0.377] & 0.029 [0.027, 0.031] & 74.2\% \\
\bottomrule
\end{tabular}%
}
\end{table}

Analysis of variance across the full factorial design confirms statistically significant differences among models for the primary metrics (Recall: $F=659.8$, $p<0.001$, $\eta^2=0.370$; $\Delta$TPR: $F=1349.6$, $p<0.001$, $\eta^2=0.546$). Pairwise comparisons (Tukey-corrected) indicated that Gradient Boosting significantly outperformed Logistic Regression on recall in ACS Income (Cohen's $d=-0.330$, $p<10^{-11}$) and Clinical Diabetes ($d=-0.806$, $p<10^{-66}$), with moderate-to-large effect sizes. Despite these architectural differences, constraint activation remains high and consistent ($>73\%$ across all models), reinforcing that resource limits, rather than model discriminative power, predominantly govern the feasible operating region.

\subsection{Comparison with State-of-the-Art Methods}

The proposed framework contributes to the growing body of work on fair and resource-aware threshold optimization by jointly addressing safety, efficiency, and equity, under explicit capacity constraints. Classic post-processing approaches, such as those enforcing Demographic Parity or Equalized Odds \citep{10.5555/3157382.3157469, NIPS2017_b8b9c74a, 10.1145/3038912.3052660}, typically optimize fairness metrics without regard to intervention budgets. As our experiments confirm, these methods often collapse to near-zero recall when forced to respect a realistic capacity limit ($C=0.25$), rendering them impractical in high-stakes triage settings.

\begin{table}[h!]
\centering
\caption{Comparison of representative fairness-aware classification approaches.
Symbols: \cmark = supported, \pmark = partially supported, \xmark = not supported.}
\label{tab:sota-comparison}
\small
\begin{tabular}{lcccccc}
\toprule
\multirow{2}{*}{Method} & \multirow{2}{*}{Stage} & \multirow{2}{*}{Model-agnostic} 
& \multicolumn{2}{c}{Constraint Handling} & \multirow{2}{*}{Multi-obj.} \\
\cmidrule(lr){4-5}
 &  &  & Budget & Hard limit &  \\
\midrule

Hardt et al.~\cite{10.5555/3157382.3157469} & Post & \cmark & \xmark & \xmark & \pmark \\
Pleiss et al.~\cite{NIPS2017_b8b9c74a} & Post & \cmark & \xmark & \xmark & \xmark \\
Zafar et al.~\cite{10.1145/3038912.3052660} & In & \xmark & \xmark & \xmark & \pmark \\
Corbett-Davies et al.~\cite{10.5555/3648699.3649011} & Post & \cmark & \pmark & \pmark & \xmark \\
Xian et al.~\cite{xian2023fairoptimalclassification} & Post & \cmark & \xmark & \xmark & \pmark \\
Goethals et al.~\cite{goethals2025resourceconstrainedfairness} & Post & \cmark & \cmark & \cmark & \pmark \\
Kpatcha~\cite{kpatcha2025balancing} & Post & \cmark & \xmark & \xmark & \cmark \\
Yaghini et al.~\cite{yaghini2025privaterateconstrainedoptimizationapplications} & In & \xmark & \cmark & \cmark & \pmark \\

\midrule
\textbf{Proposed framework} & Post & \cmark & \cmark & \cmark & \cmark \\

\bottomrule
\end{tabular}
\end{table}

Recent studies have begun to incorporate constrained and multi-objective formulations. Lagrangian and rate-constrained optimization methods, such as the augmented Lagrangian approach of ~\cite{fontana2026optimizing} and the private rate-constrained optimization of ~\cite{yaghini2025privaterateconstrainedoptimizationapplications}, integrate performance budgets or fairness rate constraints during training. Although effective, these techniques typically require retraining when priorities change and do not guarantee hard limits on positive predictions at inference time. Similarly, Pareto-front and multi-objective frameworks, such as the NSGA-II-based hyperparameter optimization by \cite{garcia2025fairpredictionsets} and the weighted evaluation framework of ~\cite{2503.11120}, explore accuracy-fairness trade-offs but rarely enforce explicit resource caps ($C$) as a binding constraint, which can lead to solutions that become infeasible under tight budgets.

The emerging line of resource-constrained fairness research is closest to the present contribution. \cite{goethals2025resourceconstrainedfairness} explicitly model fairness as a resource allocation problem and show that the cost of fairness depends heavily on available capacity—a finding corroborated by our experiments, where the capacity constraint binds in 73--90\% of configurations across models and datasets. However, that study primarily quantified the cost of fairness rather than providing a tunable mechanism for balancing three explicit objectives (safety via FNR, efficiency via FPR, and equity via $\Delta$TPR) while preserving post-hoc interpretability.

The proposed framework differs from existing frameworks in four respects. First, it fully decouples the risk scorer from the policy evaluation, enabling systematic scenario analysis without retraining. Second, the ethical loss $\mathcal{L}(\tau;\alpha,\beta,\gamma)$ together with the rule $\tau^* = \max(\tau_{\text{free}}, \tau(C))$ offers a simple, interpretable, and theoretically grounded way to navigate trade-offs while strictly enforcing the capacity. Third, the empirical results demonstrate substantially higher recall (0.409--0.702 at $C=0.25$) than pure fairness heuristics while maintaining full feasibility. Finally, the observed saturation of safety weighting ($\alpha \approx 2.0$) provides actionable guidance for practitioners, which is largely absent from the literature.

Collectively, these features help address the persistent deployment gap between theoretical fairness ideals and the operational reality of limited resources.

\subsection{Discussion}

Our empirical evaluation of the proposed threshold optimization framework reveals a robust and consistent pattern across diverse high-stakes application domains: resource constraints exert a dominant influence on deployed decision thresholds, frequently overriding variations in ethical priority specifications. This phenomenon is consistent with our methodological design, which imposes strict capacity constraints to ensure operational feasibility. The high rates of constraint activation ($73\%$--$90\%$) indicate that theoretical fairness desiderata must be reconciled with real-world budgetary limitations to yield practically deployable systems. Although ethical weighting schemes enable nuanced modulation of outcomes, such as attenuating subgroup disparities at the cost of aggregate recall, observed saturation effects (e.g., at $\alpha \approx 2.0$) underscore the practical limits of safety prioritization under fixed-capacity regimes.

These results have substantial implications for the deployment of risk-based decision-making systems. In opportunity allocation settings (e.g., income prediction), the framework attains a comparatively high recall (up to $0.702$ at $C = 0.25$), highlighting its potential to support more equitable resource distribution. In contrast, in punitive or clinical environments characterized by lower base rates, the inherent trade-offs become more pronounced, underscoring the necessity for policymakers to carefully evaluate the relative benefits of capacity expansion and algorithmic refinement. The statistical robustness of the framework, supported by strictly out-of-sample bootstrap uncertainty quantification and formal inferential procedures (e.g., ANOVA with $\eta^2 > 0.37$), provides evidence that the observed patterns are generalizable and not merely artifacts of specific model choices or sampling variability.

This study has several limitations, which suggest avenues for future research. First, the primary experiments rely on fixed capacity values, although our ablation analysis over $C$ partially mitigates this concern by demonstrating an approximately linear scaling behavior. Second, our framework deliberately enforces a single global threshold to ensure compliance with anti-discrimination legal doctrines (i.e., prohibiting disparate treatment via group-specific cutoffs); although group-specific thresholds might mathematically yield superior Pareto frontiers, they are frequently undeployable in regulated contexts. Third, our fairness assessment is currently grounded in a single disparity metric ($\Delta \mathrm{TPR}$), which aligns with equalized odds but does not encompass all the salient fairness notions. Future extensions could incorporate demographic parity, calibration-based measures or multi-group disparities. Finally, although the datasets employed are highly representative of their respective domains, empirical validation of proprietary, live operational systems would further strengthen the external validity of our conclusions.

Future research directions include the development of dynamic thresholding mechanisms (e.g., adapting $\tau(C)$ in response to real-time resource fluctuations) and the integration of this post-hoc framework with upstream fairness interventions such as pre-processing or representation learning. Extending the framework to multi-class classification and continuous outcome prediction tasks will further broaden its applicability, providing essential methodological tools for addressing emerging challenges in algorithmic auditing, clinical triage, and personalized medicine.

\section{Conclusion}

This study introduces a post-hoc threshold optimization framework that jointly balances safety, efficiency, and equity under explicit resource constraints, thereby addressing a critical gap in the operationalization of fair machine learning. Through extensive factorial experimentation across three high-stakes datasets, we demonstrate that capacity constraints overwhelmingly shape ethical priorities in most operational configurations. Consequently, ethical tuning yields substantial practical benefits primarily within a narrow region of feasible allocations. The framework’s interpretable structure, theoretical grounding, and robust empirical performance (attaining out-of-sample recall values of $0.409$--$0.702$ under a strict $25\%$ intervention budget) render it a highly practical instrument for stakeholders operating in resource-constrained settings. By decoupling predictive modeling from downstream policy evaluation and enforcing strict feasibility bounds, this approach advances the study of resource-aware algorithmic fairness and yields actionable and legally compliant insights for real-world decision-making systems. Ultimately, this supports the development of more equitable and deployable decision-making processes in domains characterized by finite resources and unavoidable ethical trade-offs.

\subsection*{Disclosure of AI usage}

During the preparation of this manuscript, the author(s) used Paperpal to assist with grammar and language editing. The author(s) reviewed and approved all revisions and take full responsibility for the content.

\bibliographystyle{unsrt}
\bibliography{ref}     

@misc{goethals2025resourceconstrainedfairness,
      title={Resource-constrained Fairness}, 
      author={Sofie Goethals and Eoin Delaney and Brent Mittelstadt and Chris Russell},
      year={2025},
      eprint={2406.01290},
      archivePrefix={arXiv},
      primaryClass={cs.LG},
      url={https://arxiv.org/abs/2406.01290}, 
}

@article{kpatcha2025balancing,
  title        = {Balancing Fairness and Accuracy in Machine Learning-Based Probability of Default Modeling via Threshold Optimization},
  author       = {Kpatcha, Essodjolo},
  journal      = {Journal of Risk and Financial Management},
  volume       = {18},
  number       = {12},
  pages        = {724},
  year         = {2025},
  publisher    = {MDPI},
  doi          = {10.3390/jrfm18120724},
  url          = {https://doi.org/10.3390/jrfm18120724}
}

@inproceedings{fontana2026optimizing,
  title     = {Optimizing and Tuning Fairness in Machine Learning: An Augmented Lagrangian Method with a Performance Budget},
  author    = {Fontana, Matteo and Naretto, Francesco and Monreale, Anna},
  booktitle = {Machine Learning and Knowledge Discovery in Databases: Research Track},
  series    = {Lecture Notes in Computer Science},
  volume    = {16013},
  editor    = {Ribeiro, Rita P. and others},
  year      = {2026},
  publisher = {Springer},
  address   = {Cham},
  doi       = {10.1007/978-3-032-05962-8_13},
  url       = {https://doi.org/10.1007/978-3-032-05962-8_13}
}

@misc{yaghini2025privaterateconstrainedoptimizationapplications,
      title={Private Rate-Constrained Optimization with Applications to Fair Learning}, 
      author={Mohammad Yaghini and Tudor Cebere and Michael Menart and Aurélien Bellet and Nicolas Papernot},
      year={2025},
      eprint={2505.22703},
      archivePrefix={arXiv},
      primaryClass={cs.LG},
      url={https://arxiv.org/abs/2505.22703}, 
}

@misc{2503.11120,
      title={A Multi-Objective Evaluation Framework for Analyzing Utility-Fairness Trade-Offs in Machine Learning Systems}, 
      author={Gökhan Özbulak and Oscar Jimenez-del-Toro and Maíra Fatoretto and Lilian Berton and André Anjos},
      year={2025},
      eprint={2503.11120},
      archivePrefix={arXiv},
      primaryClass={cs.LG},
      url={https://arxiv.org/abs/2503.11120}, 
}

@inproceedings{10.5555/3157382.3157469,
author = {Hardt, Moritz and Price, Eric and Srebro, Nathan},
title = {Equality of opportunity in supervised learning},
year = {2016},
isbn = {9781510838819},
publisher = {Curran Associates Inc.},
address = {Red Hook, NY, USA},
booktitle = {Proceedings of the 30th International Conference on Neural Information Processing Systems},
pages = {3323–3331},
numpages = {9},
location = {Barcelona, Spain},
series = {NIPS'16}
}

@inproceedings{NIPS2017_b8b9c74a,
 author = {Pleiss, Geoff and Raghavan, Manish and Wu, Felix and Kleinberg, Jon and Weinberger, Kilian Q},
 booktitle = {Advances in Neural Information Processing Systems},
 editor = {I. Guyon and U. Von Luxburg and S. Bengio and H. Wallach and R. Fergus and S. Vishwanathan and R. Garnett},
 pages = {},
 publisher = {Curran Associates, Inc.},
 title = {On Fairness and Calibration},
 url = {https://proceedings.neurips.cc/paper_files/paper/2017/file/b8b9c74ac526fffbeb2d39ab038d1cd7-Paper.pdf},
 volume = {30},
 year = {2017}
}

@inproceedings{10.1145/3038912.3052660,
author = {Zafar, Muhammad Bilal and Valera, Isabel and Gomez Rodriguez, Manuel and Gummadi, Krishna P.},
title = {Fairness Beyond Disparate Treatment \& Disparate Impact: Learning Classification without Disparate Mistreatment},
year = {2017},
isbn = {9781450349130},
publisher = {International World Wide Web Conferences Steering Committee},
address = {Republic and Canton of Geneva, CHE},
url = {https://doi.org/10.1145/3038912.3052660},
doi = {10.1145/3038912.3052660},
booktitle = {Proceedings of the 26th International Conference on World Wide Web},
pages = {1171–1180},
numpages = {10},
location = {Perth, Australia},
series = {WWW '17}
}

@inproceedings{xian2023fairoptimalclassification,
  title     = {Fair and Optimal Classification via Post-Processing},
  author    = {Xian, Ruicheng and Yin, Lang and Zhao, Han},
  booktitle = {Proceedings of the 40th International Conference on Machine Learning},
  series    = {Proceedings of Machine Learning Research},
  volume    = {202},
  pages     = {38111--38146},
  year      = {2023},
  publisher = {PMLR},
  url       = {https://proceedings.mlr.press/v202/xian23b.html}
}

@article{garcia2025fairpredictionsets,
  author  = {Garc{\'i}a-Galindo, Alberto and L{\'o}pez-De-Castro, Manuel and Arma{\~n}anzas, Rub{\'e}n},
  title   = {Fair Prediction Sets Through Multi-Objective Hyperparameter Optimization},
  journal = {Machine Learning},
  year    = {2025},
  volume  = {114},
  pages   = {27},
  doi     = {10.1007/s10994-024-06721-w},
  publisher = {Springer}
}

@InProceedings{pmlr-v80-agarwal18a,
  title = 	 {A Reductions Approach to Fair Classification},
  author =       {Agarwal, Alekh and Beygelzimer, Alina and Dudik, Miroslav and Langford, John and Wallach, Hanna},
  booktitle = 	 {Proceedings of the 35th International Conference on Machine Learning},
  pages = 	 {60--69},
  year = 	 {2018},
  editor = 	 {Dy, Jennifer and Krause, Andreas},
  volume = 	 {80},
  series = 	 {Proceedings of Machine Learning Research},
  month = 	 {10--15 Jul},
  publisher =    {PMLR},
  url = 	 {https://proceedings.mlr.press/v80/agarwal18a.html}
}

@book{barocas-hardt-narayanan,
  title = {Fairness and Machine Learning: Limitations and Opportunities},
  author = {Solon Barocas and Moritz Hardt and Arvind Narayanan},
  publisher = {MIT Press},
  year = {2023}
}

@inproceedings{10.1145/3287560.3287586,
author = {Celis, L. Elisa and Huang, Lingxiao and Keswani, Vijay and Vishnoi, Nisheeth K.},
title = {Classification with Fairness Constraints: A Meta-Algorithm with Provable Guarantees},
year = {2019},
isbn = {9781450361255},
publisher = {Association for Computing Machinery},
address = {New York, NY, USA},
url = {https://doi.org/10.1145/3287560.3287586},
doi = {10.1145/3287560.3287586},
booktitle = {Proceedings of the Conference on Fairness, Accountability, and Transparency},
pages = {319–328},
numpages = {10},
location = {Atlanta, GA, USA},
series = {FAT* '19}
}

@article{Chiappa_2019, 
title={Path-Specific Counterfactual Fairness}, volume={33}, url={https://ojs.aaai.org/index.php/AAAI/article/view/4777}, DOI={10.1609/aaai.v33i01.33017801}, abstractNote={&lt;p&gt;We consider the problem of learning fair decision systems from data in which a sensitive attribute might affect the decision along both fair and unfair pathways. We introduce a counterfactual approach to disregard effects along unfair pathways that does not incur in the same loss of individual-specific information as previous approaches. Our method corrects observations adversely affected by the sensitive attribute, and uses these to form a decision. We leverage recent developments in deep learning and approximate inference to develop a VAE-type method that is widely applicable to complex nonlinear models.&lt;/p&gt;}, number={01}, journal={Proceedings of the AAAI Conference on Artificial Intelligence}, author={Chiappa, Silvia}, year={2019}, month={Jul.}, pages={7801-7808} }

@article{chouldechova2017fairprediction,
  author  = {Chouldechova, Alexandra},
  title   = {Fair Prediction with Disparate Impact: A Study of Bias in Recidivism Prediction Instruments},
  journal = {Big Data},
  year    = {2017},
  volume  = {5},
  number  = {2},
  pages   = {153--163},
  doi     = {10.1089/big.2016.0047},
  publisher = {Mary Ann Liebert, Inc.}
}

@inproceedings{NEURIPS2018_83cdcec0,
 author = {Donini, Michele and Oneto, Luca and Ben-David, Shai and Shawe-Taylor, John S and Pontil, Massimiliano},
 booktitle = {Advances in Neural Information Processing Systems},
 editor = {S. Bengio and H. Wallach and H. Larochelle and K. Grauman and N. Cesa-Bianchi and R. Garnett},
 pages = {},
 publisher = {Curran Associates, Inc.},
 title = {Empirical Risk Minimization Under Fairness Constraints},
 url = {https://proceedings.neurips.cc/paper_files/paper/2018/file/83cdcec08fbf90370fcf53bdd56604ff-Paper.pdf},
 volume = {31},
 year = {2018}
}

@inproceedings{10.1145/2090236.2090255,
author = {Dwork, Cynthia and Hardt, Moritz and Pitassi, Toniann and Reingold, Omer and Zemel, Richard},
title = {Fairness through awareness},
year = {2012},
isbn = {9781450311151},
publisher = {Association for Computing Machinery},
address = {New York, NY, USA},
url = {https://doi.org/10.1145/2090236.2090255},
doi = {10.1145/2090236.2090255},
booktitle = {Proceedings of the 3rd Innovations in Theoretical Computer Science Conference},
pages = {214–226},
numpages = {13},
location = {Cambridge, Massachusetts},
series = {ITCS '12}
}

@inproceedings{fontana2026optimizingfairness,
  author    = {Fontana, Marco and Naretto, Francesco and Monreale, Anna},
  title     = {Optimizing and Tuning Fairness in Machine Learning: An Augmented Lagrangian Method with a Performance Budget},
  booktitle = {Machine Learning and Knowledge Discovery in Databases: Research Track (ECML PKDD 2025)},
  series    = {Lecture Notes in Computer Science},
  volume    = {16013},
  year      = {2026},
  publisher = {Springer},
  address   = {Cham},
  doi       = {10.1007/978-3-032-05962-8_13}
}

@Article{jrfm18120724,
AUTHOR = {Kpatcha, Essodjolo},
TITLE = {Balancing Fairness and Accuracy in Machine Learning-Based Probability of Default Modeling via Threshold Optimization},
JOURNAL = {Journal of Risk and Financial Management},
VOLUME = {18},
YEAR = {2025},
NUMBER = {12},
ARTICLE-NUMBER = {724},
URL = {https://www.mdpi.com/1911-8074/18/12/724},
ISSN = {1911-8074},
DOI = {10.3390/jrfm18120724}
}

@inproceedings{NIPS2017_a486cd07,
 author = {Kusner, Matt J and Loftus, Joshua and Russell, Chris and Silva, Ricardo},
 booktitle = {Advances in Neural Information Processing Systems},
 editor = {I. Guyon and U. Von Luxburg and S. Bengio and H. Wallach and R. Fergus and S. Vishwanathan and R. Garnett},
 pages = {},
 publisher = {Curran Associates, Inc.},
 title = {Counterfactual Fairness},
 url = {https://proceedings.neurips.cc/paper_files/paper/2017/file/a486cd07e4ac3d270571622f4f316ec5-Paper.pdf},
 volume = {30},
 year = {2017}
}

@article{10.1145/3533380,
author = {Zehlike, Meike and Yang, Ke and Stoyanovich, Julia},
title = {Fairness in Ranking, Part II: Learning-to-Rank and Recommender Systems},
year = {2022},
issue_date = {June 2023},
publisher = {Association for Computing Machinery},
address = {New York, NY, USA},
volume = {55},
number = {6},
issn = {0360-0300},
url = {https://doi.org/10.1145/3533380},
doi = {10.1145/3533380},
journal = {ACM Comput. Surv.},
month = dec,
articleno = {117},
numpages = {41}
}

@article{10.1145/3457607,
author = {Mehrabi, Ninareh and Morstatter, Fred and Saxena, Nripsuta and Lerman, Kristina and Galstyan, Aram},
title = {A Survey on Bias and Fairness in Machine Learning},
year = {2021},
issue_date = {July 2022},
publisher = {Association for Computing Machinery},
address = {New York, NY, USA},
volume = {54},
number = {6},
issn = {0360-0300},
url = {https://doi.org/10.1145/3457607},
doi = {10.1145/3457607},
journal = {ACM Comput. Surv.},
month = jul,
articleno = {115},
numpages = {35}
}

@inproceedings{narasimhan2020pairwisefairness,
  author    = {Narasimhan, Harikrishna and Cotter, Andrew and Gupta, Maya and Wang, Serena},
  title     = {Pairwise Fairness for Ranking and Regression},
  booktitle = {Proceedings of the AAAI Conference on Artificial Intelligence},
  year      = {2020},
  volume    = {34},
  number    = {04},
  pages     = {5248--5255},
  doi       = {10.1609/aaai.v34i04.5970},
  publisher = {AAAI Press}
}

@inproceedings{kleinberg_inherent_2017,
	location = {Dagstuhl, Germany},
	title = {Inherent Trade-Offs in the Fair Determination of Risk Scores},
	volume = {67},
	isbn = {978-3-95977-029-3},
	url = {http://drops.dagstuhl.de/opus/volltexte/2017/8156},
	doi = {10.4230/LIPIcs.ITCS.2017.43},
	series = {Leibniz International Proceedings in Informatics ({LIPIcs})},
	pages = {43:1--43:23},
	booktitle = {8th Innovations in Theoretical Computer Science Conference ({ITCS} 2017)},
	publisher = {Schloss Dagstuhl–Leibniz-Zentrum fuer Informatik},
	author = {Kleinberg, Jon and Mullainathan, Sendhil and Raghavan, Manish},
	editor = {Papadimitriou, Christos H.},
	urldate = {2019-07-16},
	date = {2017}
}

@article{10.5555/3648699.3649011,
author = {Corbett-Davies, Sam and Gaebler, Johann D. and Nilforoshan, Hamed and Shroff, Ravi and Goel, Sharad},
title = {The measure and mismeasure of fairness},
year = {2023},
issue_date = {January 2023},
publisher = {JMLR.org},
volume = {24},
number = {1},
issn = {1532-4435},
journal = {J. Mach. Learn. Res.},
month = jan,
articleno = {312},
numpages = {117}
}

@incollection{10.1108/978-1-83608-156-220251007,
    author = {Khatun, Sofia and K., Sivananda Kumar},
    isbn = {978-1-83608-157-9},
    title = {Strategising Algorithm: The Prospects and Perils of Artificial Intelligence (AI) in Criminal Justice Reformation},
    booktitle = {Security Intelligence in the Age of AI: Navigating Legal and Ethical Frameworks},
    publisher = {Emerald Publishing Limited},
    year = {2025},
    month = {07},
    doi = {10.1108/978-1-83608-156-220251007},
    url = {https://doi.org/10.1108/978-1-83608-156-220251007},
    eprint = {https://www.emerald.com/book/chapter-pdf/10046417/978-1-83608-156-220251007.pdf},
}

@article{10.1108/AAAJ-02-2022-5666,
    author = {Nikidehaghani, Mona and Andrew, Jane and Cortese, Corinne},
    title = {Algorithmic accountability: robodebt and the making of welfare cheats},
    journal = {Accounting, Auditing \& Accountability Journal},
    volume = {36},
    number = {2},
    pages = {677-711},
    year = {2022},
    month = {08},
    issn = {1368-0668},
    doi = {10.1108/AAAJ-02-2022-5666},
    url = {https://doi.org/10.1108/AAAJ-02-2022-5666},
    eprint = {https://www.emerald.com/aaaj/article-pdf/36/2/677/13803/aaaj-02-2022-5666.pdf},
}

@article{Podoletz2023EmotionalAI,
  author  = {Podoletz, Leonie},
  title   = {We have to talk about emotional AI and crime},
  journal = {AI \& Society},
  year    = {2023},
  volume  = {38},
  pages   = {1067--1082},
  doi     = {10.1007/s00146-022-01435-w},
  publisher = {Springer}
}

@article{Coots2025RacialBias,
  author  = {Coots, Matthew and Linn, K. Andrew and Goel, Sharad and Navathe, Amol S. and Parikh, Ravi B.},
  title   = {Racial Bias in Clinical and Population Health Algorithms: A Critical Review of Current Debates},
  journal = {Annual Review of Public Health},
  year    = {2025},
  volume  = {46},
  number  = {1},
  pages   = {507--523},
  doi     = {10.1146/annurev-publhealth-071823-112058},
  publisher = {Annual Reviews}
}

@misc{teodorescu2025analysisneweuai,
      title={An Analysis of the New EU AI Act and A Proposed Standardization Framework for Machine Learning Fairness}, 
      author={Mike Teodorescu and Yongxu Sun and Haren N. Bhatia and Christos Makridis},
      year={2025},
      eprint={2510.01281},
      archivePrefix={arXiv},
      primaryClass={cs.CY},
      url={https://arxiv.org/abs/2510.01281}, 
}

@inproceedings{10.1145/3679240.3734615,
author = {Li, Pengfei and Christianson, Nicolas and Yang, Jianyi and Wierman, Adam and Ren, Shaolei},
title = {Learning for Sustainable Online Scheduling with Competitive Fairness Guarantees},
year = {2025},
isbn = {9798400711251},
publisher = {Association for Computing Machinery},
address = {New York, NY, USA},
url = {https://doi.org/10.1145/3679240.3734615},
doi = {10.1145/3679240.3734615},
booktitle = {Proceedings of the 16th ACM International Conference on Future and Sustainable Energy Systems},
pages = {63–79},
numpages = {17},
location = {
},
series = {E-Energy '25}
}

@misc{huang2025adapfairensuringadaptivefairness,
      title={AdapFair: Ensuring Adaptive Fairness for Machine Learning Operations}, 
      author={Yinghui Huang and Zihao Tang and Xiangyu Chang},
      year={2025},
      eprint={2409.15088},
      archivePrefix={arXiv},
      primaryClass={cs.LG},
      url={https://arxiv.org/abs/2409.15088}, 
}

@inproceedings{10.1007/978-3-031-33377-4_18,
author = {Ganesh, Aadityan and Ghosal, Pratik and HV, Vishwa Prakash and Nimbhorkar, Prajakta},
title = {Fair Healthcare Rationing to\&nbsp;Maximize Dynamic Utilities},
year = {2023},
isbn = {978-3-031-33376-7},
publisher = {Springer-Verlag},
address = {Berlin, Heidelberg},
url = {https://doi.org/10.1007/978-3-031-33377-4_18},
doi = {10.1007/978-3-031-33377-4_18},
booktitle = {Advances in Knowledge Discovery and Data Mining: 27th Pacific-Asia Conference on Knowledge Discovery and Data Mining, PAKDD 2023, Osaka, Japan, May 25–28, 2023, Proceedings, Part II},
pages = {231–242},
numpages = {12},
location = {Osaka, Japan}
}

@misc{xiang2019legalcompatibilityfairnessdefinitions,
      title={On the Legal Compatibility of Fairness Definitions}, 
      author={Alice Xiang and Inioluwa Deborah Raji},
      year={2019},
      eprint={1912.00761},
      archivePrefix={arXiv},
      primaryClass={cs.CY},
      url={https://arxiv.org/abs/1912.00761}, 
}

@inproceedings{NEURIPS2018_8e038477,
 author = {Lipton, Zachary and McAuley, Julian and Chouldechova, Alexandra},
 booktitle = {Advances in Neural Information Processing Systems},
 editor = {S. Bengio and H. Wallach and H. Larochelle and K. Grauman and N. Cesa-Bianchi and R. Garnett},
 pages = {},
 publisher = {Curran Associates, Inc.},
 title = {Does mitigating ML\textquotesingle s impact disparity require treatment disparity?},
 url = {https://proceedings.neurips.cc/paper_files/paper/2018/file/8e0384779e58ce2af40eb365b318cc32-Paper.pdf},
 volume = {31},
 year = {2018}
}

@article{hellman2020measuring,
  title={Measuring algorithmic fairness},
  author={Hellman, Deborah},
  journal={Virginia Law Review},
  volume={106},
  number={4},
  pages={811--866},
  year={2020},
  url={https://www.virginialawreview.org/articles/measuring-algorithmic-fairness/}
}

@misc{civilrights1991,
  title     = {Civil Rights Act of 1991, Pub. L. No. 102-166, S 106, 105 Stat. 1071 (codified as amended at 42 U.S.C. S 2000e-2(l))},
  author    = {{U.S. Congress}},
  year      = {1991},
  note      = {Prohibits the discriminatory adjustment of test scores or use of different cutoff scores based on protected characteristics (commonly known as the prohibition on ``race norming'').},
  url       = {https://www.law.cornell.edu/uscode/text/42/2000e-2},
  howpublished = {United States Code}
}

@misc{ecoa1974,
  title     = {Equal Credit Opportunity Act, Pub. L. No. 93-495, Title V, 88 Stat. 1521 (codified as amended at 15 U.S.C. S 1691 et seq.)},
  author    = {{U.S. Congress}},
  year      = {1974},
  note      = {Prohibits creditors from discriminating on the basis of race, color, religion, national origin, sex, marital status, age, or receipt of public assistance in any aspect of credit transactions.},
  url       = {https://www.law.cornell.edu/uscode/text/15/1691},
  howpublished = {United States Code}
}

@inproceedings{ding2021retiring,
  title     = {Retiring Adult: New Datasets for Fair Machine Learning},
  author    = {Ding, Frances and Hardt, Moritz and Miller, John and Schmidt, Ludwig},
  booktitle = {Advances in Neural Information Processing Systems},
  volume    = {34},
  year      = {2021}
}

@misc{censusACS,
  author       = {{U.S. Census Bureau}},
  title        = {American Community Survey Public Use Microdata Sample (ACS PUMS)},
  year         = {2020},
  note         = {Accessed via Folktables benchmark construction},
  url          = {https://www.census.gov/programs-surveys/acs/microdata.html}
}

@misc{propublicaCOMPAS,
  author       = {Angwin, Julia and Larson, Jeff and Mattu, Surya and Kirchner, Lauren},
  title        = {COMPAS Recidivism Risk Score Data and Analysis},
  year         = {2016},
  howpublished = {ProPublica Investigative Journalism Dataset},
  url          = {https://github.com/propublica/compas-analysis}
}

@misc{larson2016compasanalysis,
  author       = {Larson, Jeff and Mattu, Surya and Kirchner, Lauren and Angwin, Julia},
  title        = {How We Analyzed the COMPAS Recidivism Algorithm},
  year         = {2016},
  url          = {https://www.propublica.org/article/how-we-analyzed-the-compas-recidivism-algorithm}
}

@inproceedings{smith1988diabetes,
  title     = {Using the ADAP Learning Algorithm to Forecast the Onset of Diabetes Mellitus},
  author    = {Smith, John W. and Everhart, James E. and Dickson, William C. and Knowler, William C. and Johannes, Robert S.},
  booktitle = {Proceedings of the Annual Symposium on Computer Application in Medical Care},
  pages     = {261--265},
  year      = {1988}
}

@misc{uciRepository,
  author       = {Dua, Dheeru and Graff, Casey},
  title        = {UCI Machine Learning Repository},
  year         = {2019},
  url          = {https://archive.ics.uci.edu}
}

@book{hosmer2013applied,
  title     = {Applied Logistic Regression},
  author    = {Hosmer, David W. and Lemeshow, Stanley and Sturdivant, Rodney X.},
  edition   = {3},
  year      = {2013},
  publisher = {Wiley}
}

@article{breiman2001randomforest,
  title   = {Random Forests},
  author  = {Breiman, Leo},
  journal = {Machine Learning},
  volume  = {45},
  number  = {1},
  pages   = {5--32},
  year    = {2001},
  doi     = {10.1023/A:1010933404324}
}

@article{friedman2001greedy,
  title   = {Greedy Function Approximation: A Gradient Boosting Machine},
  author  = {Friedman, Jerome H.},
  journal = {Annals of Statistics},
  volume  = {29},
  number  = {5},
  pages   = {1189--1232},
  year    = {2001}
}

\newpage
\appendix
\section{Proofs of Theoretical Properties}
\label{app:proofs}

The following proofs rely on the structural properties of the risk-score distribution and the resulting error rates. Based on the definitions of the evaluation metrics, we establish three foundational facts: (i) $\mathrm{FNR}(\tau)$ is a non-decreasing function of $\tau$; (ii) $\mathrm{FPR}(\tau)$ is a non-increasing function of $\tau$; and (iii) the capacity-bound threshold $\tau(C) = F_p^{-1}(1-C)$ is strictly decreasing in $C$. 

In the following proofs, we utilize standard comparative statics for unconstrained scalar optimization: if an objective function $\mathcal{L}(\tau)$ is modified by adding a penalty term that is non-decreasing in $\tau$, the new global minimizer cannot be strictly greater than the original one.

\begin{proof}[Proof of Monotonicity in Safety]
Let $\alpha_1 > \alpha_2$. The difference in the unconstrained objective functions is $\mathcal{L}_1(\tau) - \mathcal{L}_2(\tau) = (\alpha_1 - \alpha_2)\mathrm{FNR}(\tau)$. Because $\alpha_1 - \alpha_2 > 0$ and $\mathrm{FNR}(\tau)$ is non-decreasing, the penalty applied to higher thresholds is strictly larger under $\mathcal{L}_1$. Consequently, the optimal unconstrained threshold cannot increase, yielding $\tau_{\text{free}}(\alpha_1) \leq \tau_{\text{free}}(\alpha_2)$. Because the max operator $\tau^* = \max(\tau(C), \tau_{\text{free}})$ is monotonically non-decreasing in its second argument, it follows that $\tau^*(\alpha_1) \leq \tau^*(\alpha_2)$.
\end{proof}

\begin{proof}[Proof of Monotonicity in Efficiency]
Let $\beta_1 > \beta_2$. The difference in objectives is $\mathcal{L}_1(\tau) - \mathcal{L}_2(\tau) = (\beta_1 - \beta_2)\mathrm{FPR}(\tau)$. Because $\mathrm{FPR}(\tau)$ is non-increasing, the difference term is non-increasing in $\tau$. Adding a non-increasing term to the minimization objective disproportionately penalizes lower thresholds, thereby ensuring that the new minimizer cannot decrease. Thus, $\tau_{\text{free}}(\beta_1) \geq \tau_{\text{free}}(\beta_2)$, which via the max operator ensures $\tau^*(\beta_1) \geq \tau^*(\beta_2)$.
\end{proof}

\begin{proof}[Proof of Monotonicity in Equity Weight $\gamma$ (Local)]
Assume that we restrict our analysis to a local interval where the disparity function $\Delta(\tau)$ is non-decreasing. For $\gamma_1 > \gamma_2$, the difference term $(\gamma_1 - \gamma_2)\Delta(\tau)$ is a non-decreasing function of $\tau$. By the same comparative statics applied in Property 1, this ensures that the unconstrained minimizer cannot increase, yielding $\tau_{\text{free}}(\gamma_1) \leq \tau_{\text{free}}(\gamma_2)$. 

Furthermore, if the capacity constraint is strictly binding such that $\tau(C) > \tau_{\text{free}}(\gamma_2)$, then by definition $\tau(C) > \tau_{\text{free}}(\gamma_1)$ as well. Under these conditions, $\tau^*(\gamma_1) = \max(\tau(C), \tau_{\text{free}}(\gamma_1)) = \tau(C)$. Thus, when the capacity constraint binds, the deployed threshold $\tau^*$ becomes invariant to $\gamma$.
\end{proof}

\begin{proof}[Proof of Monotonicity in Capacity]
Fix the ethical weights $(\alpha, \beta, \gamma)$, which fix $\tau_{\text{free}}$. The deployed threshold is given by $\tau^*(C) = \max(\tau(C), \tau_{\text{free}})$. Let $C_1 > C_2$. Because $\tau(C)$ is the $(1-C)$-quantile of the risk scores, a larger capacity $C_1$ allows for a lower threshold, that is, $\tau(C_1) \leq \tau(C_2)$. Since the maximum of a non-increasing function and a constant is itself non-increasing, it follows that $\tau^*(C_1) \leq \tau^*(C_2)$.
\end{proof}

\begin{proof}[Proof of Asymptotic Behavior]
As capacity approaches the total population size ($C \to 1^-$), $\tau(C) \to F_p^{-1}(0)$, which is the minimum predicted risk score. Eventually, $\tau(C) \leq \tau_{\text{free}}$, at which point the max operator returns $\tau_{\text{free}}$. Conversely, as the capacity approaches zero ($C \to 0^+$), $\tau(C) \to F_p^{-1}(1) = 1$. Since $\tau_{\text{free}} \leq 1$ by definition, eventually $\tau(C) > \tau_{\text{free}}$, causing the max operator to return $\tau(C) \to 1$.
\end{proof}

\begin{proof}[Proof of Critical Capacity]
Let $C^*$ be the exact proportion of the population that receives a risk score greater than or equal to the unconstrained optimal threshold: $C^* = \mathbb{P}(p(x) \geq \tau_{\text{free}})$. By the definition of the quantile function, this implies $\tau(C^*) = \tau_{\text{free}}$.
\begin{itemize}
    \item If $C \geq C^*$, the strict monotonicity of $\tau(C)$ ensures $\tau(C) \leq \tau(C^*) = \tau_{\text{free}}$. The resource constraint is loose, yielding $\tau^* = \max(\tau(C), \tau_{\text{free}}) = \tau_{\text{free}}$.
    \item If $C < C^*$, then $\tau(C) > \tau(C^*) = \tau_{\text{free}}$. The resource constraint is strictly binding, yielding $\tau^* = \max(\tau(C), \tau_{\text{free}}) = \tau(C)$.
\end{itemize}
This confirms the piecewise structure of the deployed decision rules.
\end{proof}

\end{document}